\documentclass[12pt]{article}
\usepackage[margin=1in]{geometry}

\usepackage{natbib}
\usepackage[utf8]{inputenc} 
\usepackage[T1]{fontenc}    
\usepackage[colorlinks=true,allcolors=blue]{hyperref}
\usepackage{url}            
\usepackage{booktabs}       
\usepackage{amsfonts}       
\usepackage{nicefrac}       
\usepackage{microtype}      
\usepackage{mathpazo}
\usepackage{morenotations}
\title{\papertitle}
\author{Hisham Husain$^{\dagger}$
        \qquad
        Kamil Ciosek$^{\ddagger}$
        \qquad
        Ryota Tomioka$^{\ddagger}$\\\\
        $^{\dagger}$The Australian National University $\&$ Data61\\
        $^{\ddagger}$Microsoft Research Cambridge
        }
        
\begin{document}

\date{}

\maketitle

\begin{abstract}
Entropic regularization of policies in Reinforcement Learning (RL) is a commonly used heuristic to ensure that the learned policy explores the state-space sufficiently before overfitting to a local optimal policy. The primary motivation for using entropy is for exploration and disambiguating optimal policies; however, the theoretical effects are not entirely understood. In this work, we study the more general regularized RL objective and using Fenchel duality; we derive the dual problem which takes the form of an adversarial reward problem. In particular, we find that the optimal policy found by a regularized objective is precisely an optimal policy of a reinforcement learning problem under a worst-case adversarial reward. Our result allows us to reinterpret the popular entropic regularization scheme as a form of robustification. Furthermore, due to the generality of our results, we apply to other existing regularization schemes. Our results thus give insights into the effects of regularization of policies and deepen our understanding of exploration through robust rewards at large.
\end{abstract}

\section{Introduction}
Reinforcement Learning (RL) is a paradigm of algorithms which learn policies that maximize the expected discounted reward specified by a Markov Decision Process (MDP) \citep{sutton2018reinforcement}. The formulation of an MDP is well-posed with links in utility theory \citep{russell2002artificial} and specifies a reward function where the solution can be found precisely in a deterministic form. However, in practice, the reward function is typically an idealization, and it turns out that an optimal policy in this model will cope terribly when presented to unseen or uncertain situations. Intuitively, it is anticipated that there exist multiple policies that are near-optimal to this reward yet exhibit more robust and diversified behaviour. In particular, having multiple solutions of this form would be preferred since they can help the practitioner in understanding the environment and problem better. 
\begin{figure*}
    \centering
    \includegraphics[scale=0.78]{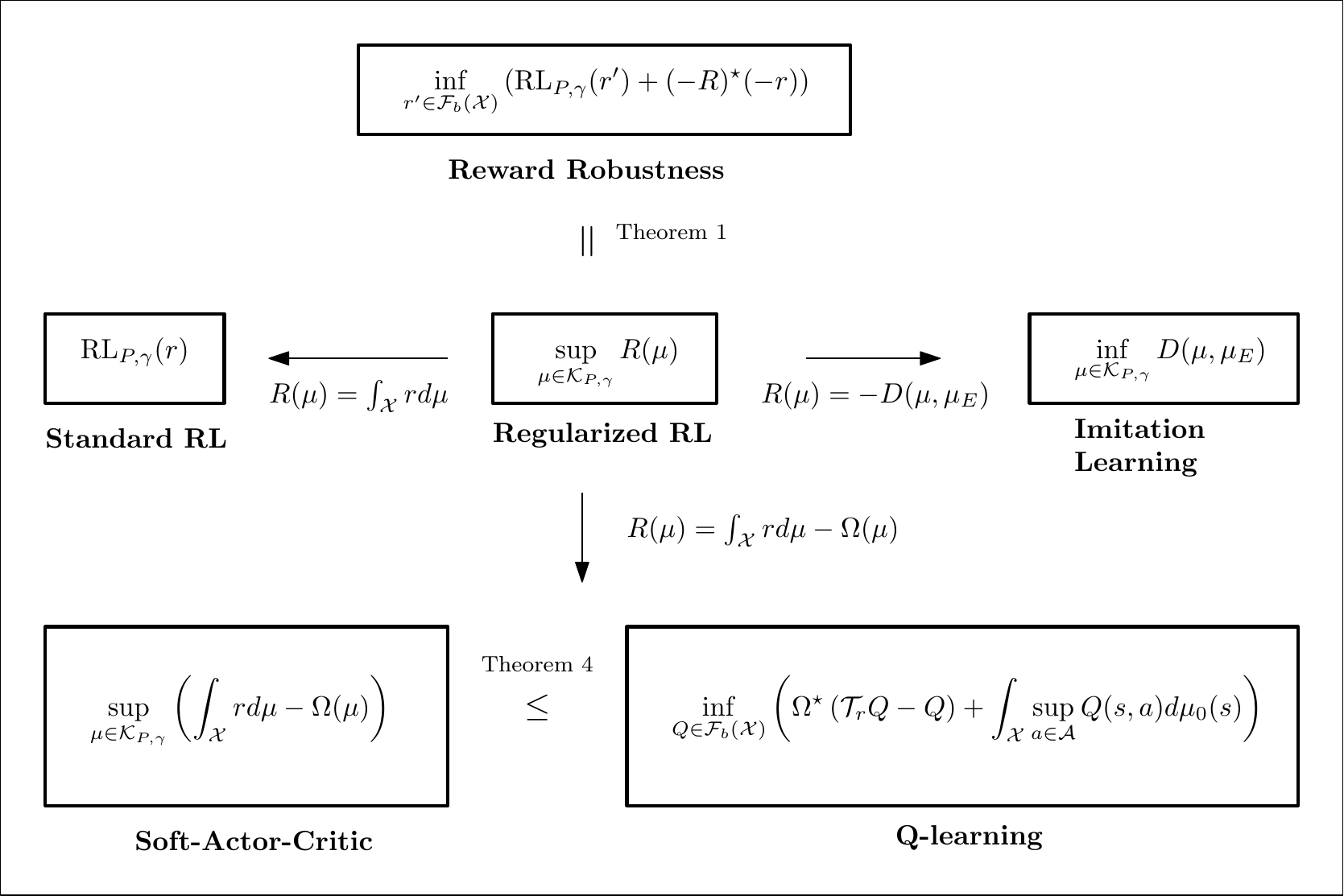}
    \caption{Our main is to provide a unified view of existing objectives in Reinforcement Learning and relate them to a reward robustness problem as highlighted above through Theorem 1.  Additionally, we show another link between regularized policies and Q-learning in Theorem 4.}
    \label{fig:contributions}
\end{figure*}
Finding near-optimal policies in this sense requires balancing between ensuring that the policy is optimal for the given reward and demonstrates some form of robustness or diversity. This is commonly recollected as the \textit{exploration} vs \textit{exploitation} trade-off. One of the most effective ways in ensuring this balance is by altering the objective of the MDP to include a form of penalty so that the resulting policy reflects characteristics of diversified behaviour. Causal entropy \citep{ziebart2010modeling} is a popular example of this, where the policy is penalized for being deterministic in favour of exploration and disambiguating optimal policies. This has lead to the MaxEnt framework \citep{haarnoja2018soft} and shown compelling relations to probabilistic inference  \citep{dayan1997using, neumann2011variational, todorov2007linearly, kappen2005path, toussaint2009robot, rawlik2013stochastic, theodorou2010generalized, ziebart2010modeling} whilst maintaining empirically superior performance on several tasks \citep{haarnoja2018soft, haarnoja2018composable}, including robustness in the face of uncertainty \citep{haarnoja2018learning}. In the case where the reward function is not specified, the entropy alone as an objective is also prevalent to ensure exploration \citep{hazan2019provably}. Similar forms of regularization have appeared in \cite{wu2019behavior}, which ensure that the policy is stabilized in accordance with a pre-determined behaviour and other forms of diversifying schemes using policy regularization have been developed in \citep{hong2018diversity}. Furthermore, the benefits of regularizers have also been observed in adversarial imitation learning methods \citep{ho2016generative,li2017infogail}.

While the empirical success should rejoice, it is somewhat unsettling that changing the objective deviates from the MDP set-up, which was initially motivated through the axioms of utility theory \citep{russell2002artificial}. In particular, it is not clear what kind of policy these regularized objectives are learning from the perspective of the original reward maximization problems, especially since it is apparent that regularized policies pose successfully in these schemes. On this front, there exists work that shows entropic regularization smoothens the optimization landspace \citep{ahmed2019understanding} and induces sparse policies when considering a larger class of policy regularizers \citep{yang2019regularized}. While these works advocate the effects of policy regularization, the benefits of regularization from an accuracy or robustness perspective and not very well understood. This is especially relevant since in machine learning more generally, regularization has shown strong links to generalization and robustness \citep{duchi2016statistics,sinha2017certifiable,husain2020distributional}. The first attempt is \citep{eysenbach2019if}, which shows that MaxEnt performs explicitly well on a robust reward problem. This approach however, is limited to only the MaxEnt and cannot apply to other schemes such as regularized imitation learning.

In this work, we tackle this precisely and focus on the problem specified by finding a policy that maximizes an objective $R$ that is concave in the space of state-action visitation distributions. This objective includes the standard reward objective and subsumes other popular objectives such as the MaxEnt framework and imitation learning. Our main insight is that the policy learned using a concave objective $R$ is \textit{robust} against rewards chosen by an adversary, where $ R$ determines the nature of the adversary. We find that the policy is precisely a maximizer against the worst-case reward $r'$. Moreover, we characterize the analytic form of $r'$ (using a technical assumption on $R$), which delivers more insight onto the nature of robustness. Our results thus allow us to reinterpret entropic regularization and exploration more generally as a robustifying mechanism and add to the advocation for using such methods in practice. In summary, our contributions are 
\begin{enumerate}
    \item A duality result linking generalized RL objectives as adversarial reward problems, which allows us to reinterpret the extant MaxEnt framework, among others, as a robustifying mechanism.
    \item Characterization of the adversarial reward solved by these regularized policy objectives. In doing so, we derive a generalized value function interpretation of entropic regularization.
    \item A primal-dual link between the regularized policy objective and Q-learning loss. This allows us to reinterpret the mean-squared error Q-learning as a form regularization of policies and robustification against rewards in light of our main result.
    \item Deriving the robust-reward problem for other popular frameworks such as imitation learning and model-free entropic optimization. This allows us to compare and unify these separate problems under reward-robustness. We illustrate this diagrammatically in Figure \ref{fig:contributions} 
\end{enumerate}
\section{Preliminaries}
\paragraph{Reinforcement Learning}
We use a compact set $\mathcal{S}$ to denote the state space, $\mathcal{A}$ the action space and set $\mathcal{X} = \mathcal{S} \times \mathcal{A}$. We assume these spaces are Polish and furthermore use $\mathscr{P}(\mathcal{S})$, $\mathscr{P}(\mathcal{A})$ and $\mathscr{P}(\mathcal{X})$ to denote the set of Borel probability measures. Similarly, we use $\mathcal{F}_b(\mathcal{S})$, $\mathcal{F}_b(\mathcal{A})$ and $\mathcal{F}_b(\mathcal{X})$ to denote the set of bounded and measurable functions on the sets $\mathcal{S}, \mathcal{A}$ and $\mathcal{X}$ respectively. A reward function is a mapping $r: \mathcal{X} \to \mathbb{R}$, a transition kernel is specified as $P: \mathcal{X} \to \mathscr{P}(\mathcal{S})$ and a policy is a mapping $\pi: \mathcal{S} \to \mathscr{P}(\mathcal{A})$. Let $\gamma > 0$ be an implicit fixed discount parameter. It can be shown that each $\mathcal{S}$, $\mathcal{A}$, $P$, initial distribution $\mu_0$ and policy $\pi$ uniquely define a Markov chain $\braces{(S_t,A_t)}_{t=1}^{\infty} \subseteq \mathcal{X}$. We denote the underlying probability space as $(\mathcal{X}, \mathscr{T}, P_{\mu_0, \pi})$ where $P_{\mu_0, \pi} \in \mathscr{P}(\mathcal{X})$ is referred to as the state-action visitation distribution. We refer the reader to \citep[Chapter~3]{meyn2012markov} and \citep[Chapter~2]{revuz2008markov} for more detailed constructions. The goal in RL is to find a policy that maximizes expected return over the state-action pairs visited, which can be concretely summarized in the optimization problem:
\begin{align}
    \sup_{\pi: \mathcal{S} \to \mathscr{P}(\mathcal{A})} \E_{P_{\mu_0, \pi}(s,a)} \left[ r(s,a) \right]. \label{RL-def}
\end{align}
This objective is linear in the space of state-action visitation distributions and thus is equivalent to the linear program $\max_{\mu \in \mathcal{K}_{P,\gamma}}\int_{\mathcal{X}} r(s,a) d\mu(s,a)$ where
\begin{align*}
    \mathcal{K}_{P,\gamma} = \Bigg\{ \Bigg. \mu \in \mathscr{P}(\mathcal{X}) : &\int_{\mathcal{A}} \mu(s,a) da = (1 - \gamma) \mu_0(s) + \gamma \int_{\mathcal{X}} P(s \mid s',a') d\mu(s',a') \Bigg. \Bigg\}.
\end{align*}
In particular, for any policy $\pi$, we have that $P_{\mu_0, \pi} \in \mathcal{K}_{P,\gamma}$ and that for any element $\mu \in \mathcal{K}_{P,\gamma}$, we can construct the corresponding policy $\pi_{\mu}(s) = \mu(s,a) / \int_{\mathcal{A}} \mu(s,a) da$. We introduce notation to formally write this since it will serve useful for the remainder of the paper.
\begin{definition}
For a reward function $r: \mathcal{X} \to \mathbb{R}$, we define
\begin{align*}
    &\operatorname{RL}_{P, \gamma}(r) := \sup_{\mu \in \mathcal{K}_{P,\gamma}} \int_{\mathcal{X}} r(s,a) d\mu(s,a)\\
    &M_{P,\gamma}(r) := \argsup_{\mu \in \mathcal{K}_{P,\gamma}} \int_{\mathcal{X}} r(s,a) d\mu(s,a)
\end{align*}
\end{definition}
In the above, $\operatorname{RL}_{P,\gamma}(r)$ is the same as \eqref{RL-def} and represents the maximum expected reward possible under an environment $P$, discount factor $\gamma$ and reward function $r$. The set $M_{P,\gamma}(r) \subseteq \mathscr{P}(\mathcal{X})$ represent the solutions that achieve maximal expected reward. 

\paragraph{Convex Analysis and Legendre-Fenchel Duality}
We use $\mathscr{B}(\mathcal{X})$ to denote the set of finitely-additive measures and denote its topological dual to be $\mathcal{F}_b(\mathcal{X})$, the set of measurable and bounded functions mapping from $\mathcal{X}$ to $\mathbb{R}$. For any functional $F: \mathscr{B}(\mathcal{X}) \to \mathbb{R}$, we define the Legendre-Fenchel dual, for any $h \in \mathcal{F}_b(\mathcal{X})$ as
\begin{align*}
    F^{\star}(h) = \sup_{\mu \in \mathscr{B}(\mathcal{X})} \bracket{ \int_{\mathcal{X}}h(x) d\mu(x) - F(\mu) }.
\end{align*}
For a set of functions $\mathcal{F} \subseteq \mathcal{F}_b(\mathcal{X})$, we use $\iota_{\mathcal{F}}(h)$ to denote the convex indicator function defined which is $0$ if $h \in \mathcal{F}$ and $+\infty$ otherwise. For any two measures $\mu,\nu \in \mathscr{B}(\mathcal{X})$, we define the $f$-divergence between $\mu$ and $\nu$ to be $D_f(\mu,\nu) = \int_{\mathcal{X}}f(d\mu / d\nu) d\nu - \int_{\mathcal{X}} d\nu + 1$ where $f:\mathbb{R} \to (-\infty, \infty]$ is a lower semicontinuous convex function with $f(1) = 0$. In particular, the setting of $f(t) = t \log t$ is the popular Kullback-Leiber divergence, which we denote by  $\operatorname{KL}(\mu,\nu) = D_f(\mu,\nu)$.
\section{Related Work}
Our main contribution is a reinterpretation of regularized policy maximization as robustifying mechanisms and so we discuss developments at understanding these methods along with similar results existing in machine learning at large. The idea of using causal entropy \citep{ziebart2010modeling} is guided by the intuition of encouraging curious and diversified behavior. Further developed in \citep{haarnoja2018soft}, empirical success of using this penalty has been apparent. In particular, regularized policies unlike standard policies have illustrated robust behavior in the face of uncertainty and diversified behavior in finite sample schemes. Despite the empirical success, there is not much work studying these benefits from a formal perspective. The main existing results show that regularized objectives include smoothen the optimization landscape \citep{ahmed2019understanding} and yield sparse policies \citep{yang2019regularized}. \citep{eysenbach2019if} focuses on the MaxEnt framework and relates the optimal policy to solving a variable reward problem, which is line with our findings. Their results in contrast to ours, cannot be applied to other policy regularizers or other schemes that use causal entropy in the absence of reward functions such as adversarial imitation learning \citep{li2017infogail}.

In the realm of machine learning more generally, regularization has been principally established as a robustifying strategy. In supervised learning, various forms of robustness have shown connections to a number of regularization penalties such as Lipschitzness \citep{blanchet2019quantifying, sinha2017certifiable, cranko2020generalised, husain2020distributional}, variance \citep{duchi2016statistics} and Hilbert space norms \citep{staib2019distributionally}. In Optimal Transport (OT), it has also been shown that entropic regularization is linked to ground cost robustness \citep{paty2020regularized}. Our result thus extends and develops these narratives for RL. \citep{zhang2020variational} also uses technical tools similar to our work such as Fenchel duality however for their purposes and findings are for quite different purposes. 
\section{Reward Robust Reinforcement Learning}
We will be focusing on the problem specified by 
\begin{align*}
    \sup_{\mu \in \mathcal{K}_{P,\gamma}} R(\mu),
\end{align*}
where $R: \mathscr{B}(\mathcal{X}) \to \mathbb{R}$ is a concave upper semicontinuous function. Note that when a reward function $r: \mathcal{X} \to \mathbb{R}$ is given, setting $R(\mu) = \int_{\mathcal{X}} r(x) d\mu(x)$ recovers the standard maximum expected reward problem. Furthermore, the above subsumes other developments of RL in the case where the reward is unknown and $R$ is chosen to be the entropy \citep{hazan2019provably} or imitation learning when $R(\mu) = -D(\mu,\mu_E)$ where $\mu_E$ is some expert demonstration and $D$ is a divergence between probability measures \citep{ghasemipour2019divergence}. We present the main result which shows the above as a reward robust RL problem.
\begin{theorem}\label{value-theorem}
For any concave upper semicontinuous function $R: \mathscr{B}(\mathcal{X}) \to \mathbb{R}$, we have
\begin{align*}
    \sup_{\mu \in \mathcal{K}_{P,\gamma}} R(\mu) = \inf_{r' \in \mathcal{F}_b(\mathcal{X})}\bracket{\operatorname{RL}_{P,\gamma}\bracket{r'} + (-R)^{\star}\bracket{-r'}  }
\end{align*}
\end{theorem}
\begin{proof}\textbf{(Sketch)} 
The key part of the proof is to rewrite $R$ in terms of the convex conjugate of $-R$, which is well-defined since $-R$ is lower semicontinuous and convex, by assumptions on $R$. The proof then concludes by moving the supremum over $\mu$ inside by an application of a generalized minimax theorem.
\end{proof}
The key point from the above is that the value of the maximal policy over $R$ is exactly equal to the problem of finding an adversarial reward. In particular, the adversarial reward problem seeks to find a reward $r'$ that makes the maximally achievable reward $\operatorname{RL}_{P,\gamma}$ as small as possible while paying the penalty $(-R)^{\star}(-r')$, where $(-R)^{\star}$ is a convex function. We present now a result linking the optimal $\mu$ and adversarial reward $r'$ above which allows us to give concrete insight.
\begin{theorem}\label{variable-theorem}
Let $\mu^{*}$ and $r^{*}$ be the optimal solution to the problems specified in Theorem \ref{value-theorem}, then we have that $\mu^{*} \in M_{P,\gamma}\bracket{r^{*}}$.
\end{theorem}
This result tell us that an optimal policy found by solving the regularized objective is precisely an optimal policy of the Reinforcement Learning problem specified by the adversarial reward $r^{*}$. This is particularly striking since it tells us that though we are maximizing some concave $R$, which may be motivated for separate purposes, we can always guarantee that the policy learned is optimal for some reward $r'$ in the axiomatic utility theory sense. In particular, this reward $r^{*}$ is chosen to be the worst-case for this environment. The strength of robustness and nature of the adversarial reward clearly depends on the choice of $R$, as this is what budgets the adversarial reward $r'$. We will show that under a technical assumption on $R$, we can characterize the form $r^{*}$ takes, which happens to depend on a single state-dependent mapping $V \in \mathcal{F}_b(\mathcal{S})$. The particular technical assumption on $(-R)^{\star}$ is that it is \textit{increasing} by which we mean $r(x) \geq r'(x)$ for every $x \in \mathcal{X}$ implies $(-R)^{\star}(r) \geq (-R)^{\star}(r')$. We first introduce a result.
\begin{theorem}\label{generalized-valueproblem}
Suppose $R$ is concave upper semicontinuous and let $\mathscr{I}$ be the value of the optimization problem
\begin{align}
     &\inf_{V \in \mathcal{F}_b(\mathcal{S}), r \in \mathcal{F}_b(\mathcal{X})}\bracket{ (1 - \gamma) \int_{\mathcal{S}}V(s) d\mu_0(s) + (-R)^{\star}(-r) }, \label{eq:generalized-valueproblem}\\ &\operatorname{s.t. }V(s) \geq  r(s,a) + \gamma \int_{\mathcal{S}} V(s') dP(s' \mid s,a). \nonumber
\end{align}
It then holds that $\mathscr{I} = \sup_{\mu \in \mathcal{K}_{P,\gamma}} R(\mu)$.
\end{theorem}
It should be first noted that the above is a strong duality Theorem and indeed is a generalized version of the standard linear programming duality between policy maximization and value function minimization as described in \citep{agarwal2019reinforcement}, which is recovered when $R(\mu) = \int_{\mathcal{X}} r(x) d\mu(x)$ for some reward $r$. We will now show that the optimal value function of this objective gives the optimal reward. In particular, note that by solving the above constraint for the reward yields
\begin{align}\label{reward-valueshape}
   r_V(s,a) := V(s) - \gamma \cdot \int_{\mathcal{S}} V(s') dP(s' \mid s,a).
\end{align}
We then have the following result 
\begin{lemma}\label{lemma:optVandoptr}
Suppose $(-R)^{\star}$ is increasing and $V^{*}$ is the optimal solution of \eqref{eq:generalized-valueproblem} then $r_{V^{*}}$ is the optimal adversarial reward.
\end{lemma}
The main consequence of the above Lemma is that it characterizes the shape of the adversarial reward chosen. In particular, it tells us that as long as as $R$ satisfies the technical assumption ($(-R)^{\star}$ is increasing), the adversarial reward will be of the form $r_V$ for some $V$. This is insightful since it tells us that the adversarial reward relates rewards between states through the dynamics of $P$. For example, note that if a particular state-action pair $(s,a)$ yields the same state $s$ then $r_V(s,a) = (1-\gamma)V(s)$. This technical condition on $R$ can be satisfied for any $R$ with a simple reparametrization, which we lay out in Lemma 1 in the supplementary material, and exploit when deriving $(-R)^{\star}$ for Soft-Actor-Critic. Moreover, we will show that the common choices of $R$ which are motivated for smoothing or other empirical benefits naturally satisfy this technical assumption.
\paragraph{Generalized Soft-Actor-Critic Regularization}  Consider the case of having an available reward and using a convex penalty $\Omega: \mathscr{B}(\mathcal{X}) \times \mathscr{B}(\mathcal{X}) \to \mathbb{R}$ for the policy so we select $R = R_{\Omega}$ of the form
\begin{align*}
    R_{\Omega}(\mu) = \int_{\mathcal{X}} r(s,a) d\mu(s,a) - \epsilon \cdot \Omega(\mu),
\end{align*}
for some $\epsilon > 0$. It can easily be shown (see Appendix) that $(-R)^{\star}(-r') = \epsilon \Omega^{\star}\bracket{\frac{r - r'}{\epsilon}}$, so that we have the following.
\begin{corollary}\label{value-corollary}
Let $\Omega: \mathscr{B}(\mathcal{X}) \to \mathbb{R}$ be a convex penalty then for any $\epsilon > 0$ we have
\begin{align*}
    \sup_{\mu \in \mathcal{K}_{P,\gamma}} R_{\Omega}(\mu) = \inf_{r' \in \mathcal{F}_b(\mathcal{X}) } \bracket{ \operatorname{RL}_{P,\gamma}\bracket{r'} + \epsilon \Omega^{\star}\bracket{\frac{r-r'}{\epsilon}} }.
\end{align*}
\end{corollary}
The above tells us that the adversarial reward problem pays a price for deviating from the given reward $r$ due to the second term $\epsilon \Omega^{\star}\bracket{\frac{r-r'}{\epsilon}}$. In the Soft-Actor-Critic (SAC) method, this corresponds to selecting (upto some constant) $\Omega_{\operatorname{SAC}}(\mu) = \E_{\mu(s,a)}\left[\operatorname{KL}(\pi_{\mu}(\cdot \mid s),U)\right]$, where $\pi_{\mu}$ is the policy induced by $\mu$ and $U$ is the uniform distribution over $\mathcal{A}$. We presented Corollary \ref{value-corollary} with a general $\Omega$, which we believe will be useful for future developments. In this work, we consider the causal policy entropy along with 2-Tsallis entropy in the next next section. For the SAC case, we have the following result
\begin{lemma}[Soft-Actor-Critic]\label{SAC-conjugate}
For any $\epsilon > 0$ and $r,r' \in \mathcal{F}(\mathcal{X})$, we have
\begin{align*}
    \epsilon\Omega_{\operatorname{SAC}}^{\star}\bracket{\frac{r - r'}{\epsilon}} = \epsilon \cdot \sup_{s \in \mathcal{S}} \bracket{\int_{\mathcal{X}} \exp\bracket{\frac{r(s,a) - r'(s,a)}{\epsilon}} dU(a) - 1}
\end{align*}
\end{lemma}
If one reasons about how the adversary behaves, the first incentive is to make $\operatorname{RL}_{P,\gamma}(r')$ small by selecting very small rewards across the environment. However, we can see that for the case of entropic regularization, the adversary pays a big price for selecting $r'$ to be far from the original reward $r$ for any given state.  Note that in this case, we have $(-R)^{\star}$ is increasing and so in light of the concrete insight found in Lemma \ref{lemma:optVandoptr}, we are able to reason about the SAC policy maximizing a reward of the worst-case reward of the form \eqref{reward-valueshape}. This is striking since it tells us that the adversarial reward $r'$ will respect the environment dynamics across the action space even if the ground reward $r$ does not.

\paragraph{Derivation of Q-learning through robust learning}
In this subsection, we derive Q-learning through the reward-robust RL framework. In this context, learning a policy that is robust to a small variation in the reward corresponds to allowing a small violation of the Bellman equation with respect to the original reward function. For any Q-function $Q \in \mathcal{F}_b(\mathcal{X})$, we define the bellman operator $\mathcal{T}_r: \mathcal{F}_b(\mathcal{X}) \to \mathcal{F}_b(\mathcal{X})$ as
\begin{align*}
    \mathcal{T}_r Q(s,a) = r(s,a) + \gamma \int_{\mathcal{X}}\sup_{a' \in \mathcal{A}} Q(s',a') dP(s' \mid s,a)
\end{align*}
The maximum reward problem can be restated as
\begin{align} \label{exact-Q-learning} 
    \operatorname{RL}_{P,\gamma}(r) = \inf_{Q \geq \mathcal{T}_r Q}  \int_{\mathcal{S}} \sup_{a \in \mathcal{A}} Q(s,a) d\mu_0(s),
\end{align}
where the optimal $Q^{*} \in \mathcal{F}_b(\mathcal{X})$ from the above is a contraction of $\mathcal{T}_r$ meaning that $\mathcal{T}_r Q^{*} = Q^{*}$. As it is difficult to find this contraction, one method known as \textit{deep Q-learning} tackles this by parametrizing $Q$ with a deep neural network and uses regression in the supervised learning sense to match $\mathcal{T}_r Q$ to $Q$ \citep{sutton2018reinforcement}. This will deviate from the original objective since it relaxes this constraint $Q = \mathcal{T}_r \mathcal{Q}$ into the term appearing in the objective, which will naturally introduce bias. We now show quite a remarkable connection that doing so is related to policy regularization and by virtue of Corollary \ref{value-corollary}, linked to reward robustness.
\begin{theorem}\label{thm:smoothQlearning}
For any $\epsilon > 0$ and convex $\Omega$ such that $\Omega^{\star}$ is increasing, we have
\begin{align*}
    \sup_{\mu \in \mathcal{K}_{P,\gamma}} R_{\Omega}(\mu) = &\inf_{Q \in \mathcal{F}_b(\mathcal{X})}  \Bigg ( \Bigg. \epsilon\Omega^{\star}\bracket{\frac{\mathcal{T}_r Q - Q}{\epsilon}} +  \int_{\mathcal{S}} \sup_{a \in \mathcal{A}} Q(s,a) d\mu_0(s)  \Bigg. \Bigg).
\end{align*}
\end{theorem}
We remark that the above is an inequality if $\Omega^{\star}$ is not increasing which results in \textit{weak duality}. First note that the Theorem is precisely a relaxed \textit{unconstrained} version of \textit{constraint} objective appearing in \eqref{exact-Q-learning}. The most notable aspect of this result is that it links the regularized objective to finding a Q-function that minimizes the difference in the Bellman update $\epsilon\Omega^{\star}\bracket{\frac{\mathcal{T}_r Q - Q}{\epsilon}}$, depending on the choice of $\Omega$. There exists work that show a relationship between gradients in entropy regularization and Q-learning \citep{schulman2017equivalence}, however we state a more generalized result and bridge it to reward robustness. To see how this relates to the existing losses used in Q-learning, let us consider both the finite and continuous case. In the finite case, we can pick $\Omega(\mu) = \sum_{x \in \mathcal{X}} \mu(x)^2$, which is the 2-Tsallis entropy. One can easily derive the dual $\Omega^{\star}(r) = \frac{1}{4}\sum_{x \in \mathcal{X}} r(x)^2$ and thus the right side of Theorem \ref{thm:smoothQlearning} becomes (setting $\epsilon = 1)$
\begin{align*}
    \inf_{Q \in \mathcal{F}_b(\mathcal{X})} \Bigg ( \Bigg. &\frac{1}{4}\sum_{(s,a) \in \mathcal{X}} \bracket{\mathcal{T}_rQ(s,a) - Q(s,a) }^2 +    \int_{\mathcal{S}} \sup_{a \in \mathcal{A}} Q(s,a) d\mu_0(s) \Bigg. \Bigg).
\end{align*}
The variational problem above is a regression problem between $Q$ and $\mathcal{T}_r Q$ using the squared loss, which is the typical objective in deep Q-learning. The consequence of our result is that using this particular choice of loss to learn the $Q$ function is related to learning a policy with the 2-Tsallis entropy, which is rather striking. Furthermore, the 2-Tsallis entropy behaves similar to the Shannon entropy in the sense that it is maximized when $\mu$ is uniform and minimized when $\mu$ is degenerate. In the continuous case, a buffer distribution $\nu \in \mathscr{P}(\mathcal{X})$ is used for the loss by defining the mean-squared error as $L^2$ norm with respect to $\nu$ between $\mathcal{T}_r Q$ and $Q$: given by $\nrm{\mathcal{T}_r Q - Q}^2_{L^2(\nu)}$. In this case, it can be shown that if $\Omega(\mu) =  \frac{1}{4}\int_{\mathcal{X}} \bracket{\frac{d\mu}{d\nu}}^2 d\nu$ when $\mu \ll \nu$ and $+\infty$ otherwise then $\Omega^{\star}(h) = \nrm{h}^2_{L^2(\nu)}$.
\paragraph{Imitation Learning} One method of learning a policy is to imitate expert data which comes in the form of a given distribution $\mu_E \in \mathscr{P}(\mathcal{X})$. Unlike the regularized schemes above, there is no specified reward function. Using the unified perspective provided in \citep{ghasemipour2019divergence}, where imitation learning is cast as divergence minimization, we can write these methods into our framework by selecting $R(\mu) = -D(\mu,\mu_E)$ (for each corresponding divergence). In particular, our goal is to not only derive the corresponding robust-reward problem but also show that $(-R)^{\star}$ will be increasing for these cases. We delegate the technical derivations to the Supplementary Section 1.8 and only present the results here. First, we focus on Adversarial Inverse Reinforcement Learning (AIRL) \citep{fu2017learning} selecting $R(\mu) = - \operatorname{KL}(\mu,\mu_E)$ in which case we have 
\begin{align*}
    &\sup_{\mu \in \mathcal{K}_{P,\gamma}} R(\mu)\\ &= \inf_{r' \in \mathcal{F}_b(\mathcal{X})}\bracket{\operatorname{RL}_{P,\gamma}(r') + \int_{\mathcal{X}} \exp\bracket{-r'(x)} d\mu_E(x) -1 },
\end{align*}
noting that $(-R)^{\star}$ is increasing. We show the more general result that when $R(\mu) = -D_f(\mu,\mu_E)$ where $D_f$ is an $f$-divergence then $(-R)^{\star}$ will be increasing. Using this choice of $R$ corresponds to $f$-MAX \citep{ghasemipour2019divergence}. Another method for imitation learning is to use a discriminator based divergence as employed in InfoGAIL \citep{li2017infogail}. In this setting we assume we have a distance $d: \mathcal{X} \times \mathcal{X} \to \mathbb{R}$ and denoting the Lipschitz constant of a function $h \in \mathcal{F}_{b}(\mathcal{X})$ as $\operatorname{Lip}_d(h) := \sup_{x,x' \in \mathcal{X}} \card{h(x) - h(x')}/d(x,x')$, we set
\begin{align*}
    R(\mu) = - \sup_{h : \operatorname{Lip}_d(h) \leq L} \bracket{\int_{\mathcal{X}}h(x)d\mu(x) - \int_{\mathcal{X}} h(x) d\mu_E(x) }, 
\end{align*}
where $L > 0$ is chosen as a hyperparameter. In this case, we have
\begin{align*}
    \sup_{\mu \in \mathcal{K}_{P,\gamma}} R(\mu) = \inf_{r' : \operatorname{Lip}_d(r') \leq L} \bracket{ \operatorname{RL}_{P,\gamma}(r') - \int_{\mathcal{X}} r' d\mu_E }.
\end{align*}
It is clear from the above that the adversarial reward seeks to ensure $\operatorname{RL}_{P,\gamma}$ is as low as possible while maintaining that $r'$ is large around the expert trajectory due to the second term. It should also be noted that the choice of $L$ reflects as the budget of the adversary. We do not have $(-R)^{\star}$ increasing for this choice of $R$. On the other hand, it is typical in practice that an entropy term is included in this term:
\begin{align*}
    R(\mu) =&- \sup_{h : \operatorname{Lip}_d(h) \leq L} \bracket{\int_{\mathcal{X}}h(x)d\mu(x) - \int_{\mathcal{X}} h(x) d\mu_E(x) }\\ &- \epsilon \E_{\mu(s,a)}\left[\operatorname{KL}(\pi_{\mu}(\cdot \mid s), U_{A}) \right],
\end{align*}
for some $\epsilon > 0$ where $U_{A}$ is the uniform distribution over $\mathcal{A}$. Under this setting, it turns out that $(-R)^{\star}$ is now increasing, in which case Lemma \ref{lemma:optVandoptr} applies. It is rather intriguing that the role of entropy here ensures that the reward that the InfoGAIL policy maximizes is worst-case, of high value around trajectories from the expert, and attains the familiar shape in Equation \eqref{reward-valueshape}. This further advocates for the use of entropy regularization.
\paragraph{Entropic Exploration} We now consider the case where there is no reward function or expert distribution specified and the only objective to maximize is entropy. For such a scheme, there exists efficient algorithms \citep{hazan2019provably}. More specifically, we have $R(\mu) = -\operatorname{KL}(\mu, U_{\mathcal{X}})$ where $U_{\mathcal{X}}$ is the uniform distribution over $\mathcal{X}$. We then have that
\begin{align*}
    &\sup_{\mu \in \mathcal{K}_{P,\gamma}} R(\mu)\\ &= \inf_{r' \in \mathcal{F}_b(\mathcal{X})} \bracket{ \operatorname{RL}_{P,\gamma}(r') + \int_{\mathcal{X}} \exp\bracket{-r'(x)} dU_{\mathcal{X}}(x) - 1 },
\end{align*}
and similar to the other choices of $R$, we have that $(-R)^{\star}$ is increasing. We would like to remark that if one defines $\operatorname{KL}$ to be $+\infty$ when $\mu$ is not a probability measure then $(-R)^{\star}(r) = \log \int_{\mathcal{X}} \exp (r(x)) dU_{\mathcal{X}}(x)$ \citep{ruderman2012tighter}.
\begin{figure*}
     \centering
     \begin{subfigure}[b]{0.3\textwidth}
         \centering
         \includegraphics[scale=0.65]{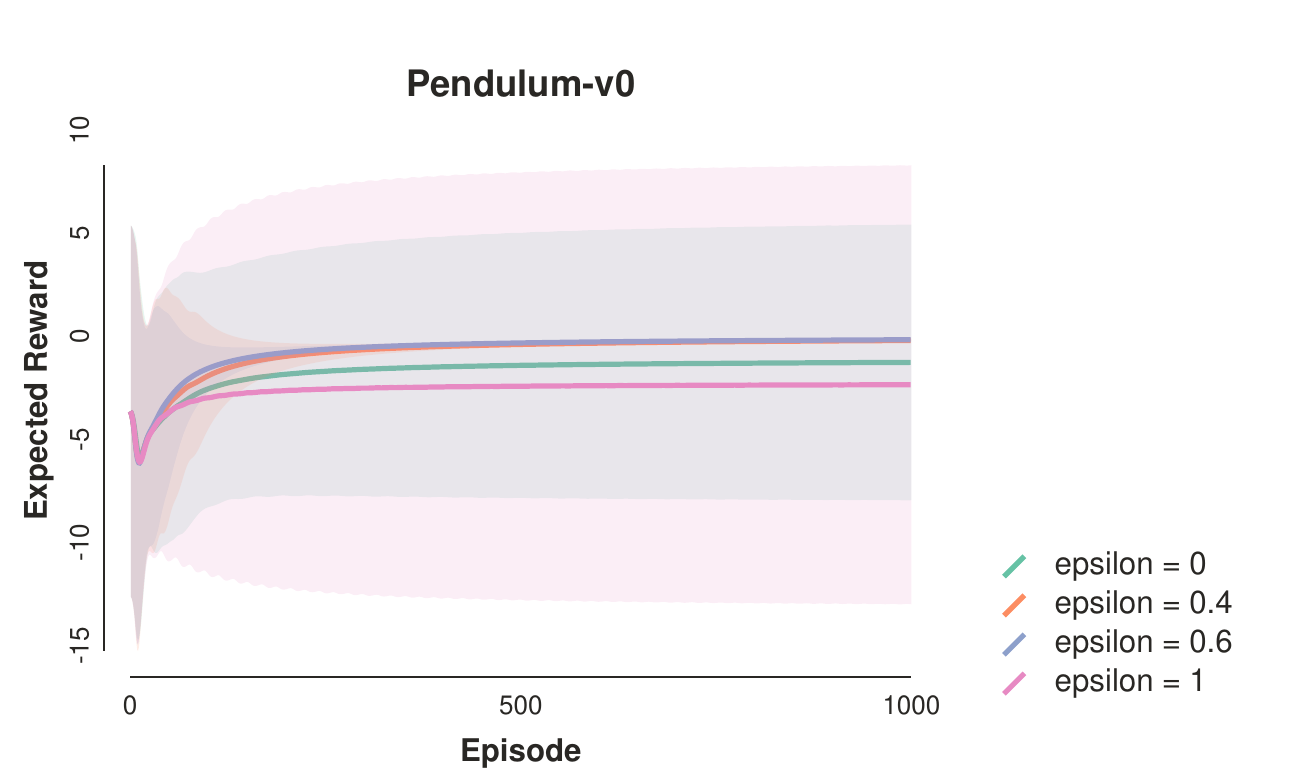}
         \label{}
     \end{subfigure}
     \hfill
     \begin{subfigure}[b]{0.5\textwidth}
         \centering
         \includegraphics[scale=0.65]{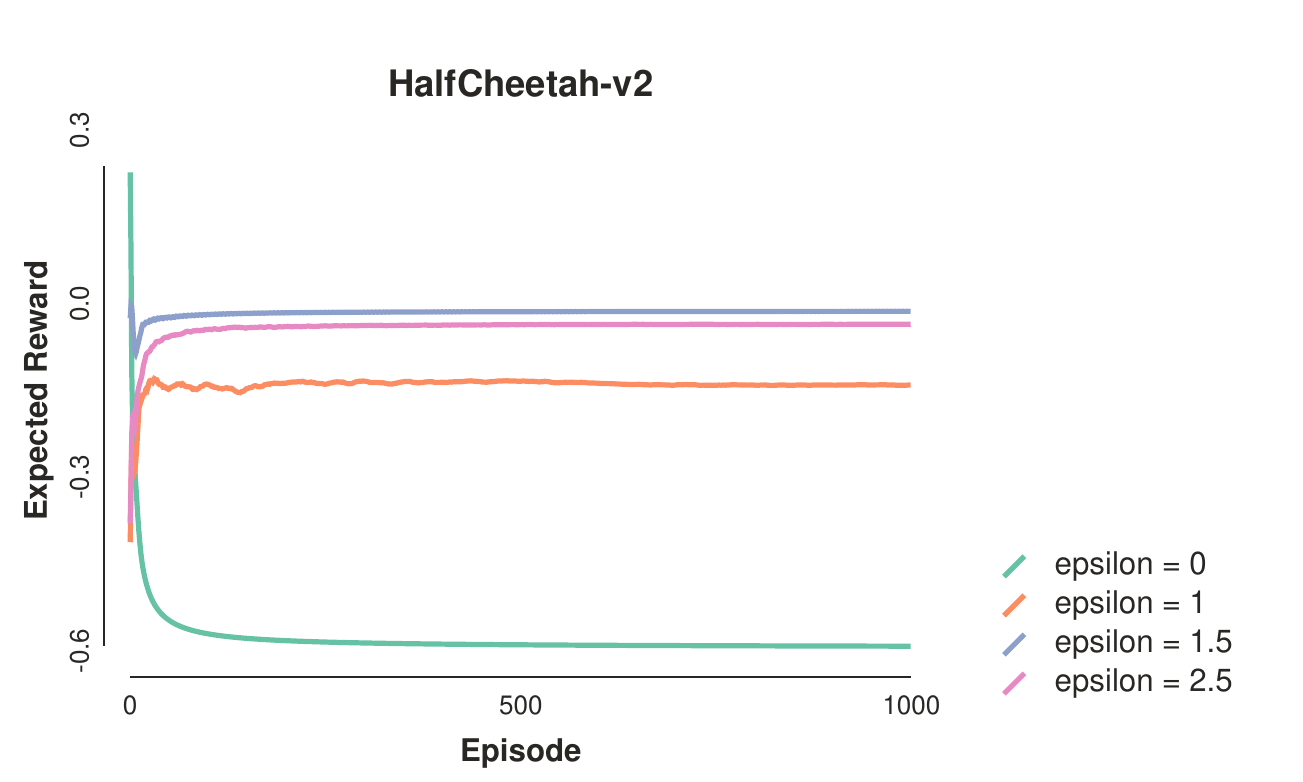}
         \label{}
     \end{subfigure}
        \caption{Expected reward over $1000$ episodes of policies returned by SAC trained on an adversarial reward $r_{\operatorname{adv}}$ and tested on the true reward using different weighting $\epsilon$ for entropy.}
        \label{fig:experiment}
\end{figure*}
\section{Experiments}
The main practical ramification of our work is to advocate the use of regularized policies by highlighting the robustification aspect, for which we derived a strong theoretical link. There exists extensive empirical evidence for which our work provides foundation for. However, we will show some brief yet illustrative examples which focus on the reward adversarial aspect of regularized policies, as illustrated by our main result Theorem 1. Our goal is thus to see the performance of regularized policies on rewards they are not trained on and analyze their behavior based on the robustness parameter $\epsilon$. First we consider the Pendulum-v0 environment and train the Soft-Actor-Critic (SAC) method on a reward that has been altered with. We do so by constructing an adversarial reward $r_{\operatorname{adv}}$ using
\begin{align*}
    r_{\operatorname{adv}} = \begin{cases} r(s,a) + \delta &\text{  if }r(s,a) \leq -5\\ r(s,a) &\text{  otherwise} \end{cases}
\end{align*}
where $\delta$ is drawn from a normal distribution centered at $5$ with variance $0.1$. In doing so, initial states of the pendulum will be favored and easier to reach however the maximal reward will still be attained at the inverted position. We train SAC for various values of $\epsilon$ and test their performance on the true reward in Figure \ref{fig:experiment} (left). We find that the effect of increasing $\epsilon$ yields better performance than no entropy however adding too much entropy (in the case of $\epsilon = 1$) damages performance. We repeat a similar experiment for HalfCheetah-v2 however using an adversarial reward specified by
\begin{align*}
    r_{\operatorname{adv}} = \begin{cases} r(s,a) + \delta &\text{  if }r(s,a) \leq 0\\ r(s,a) &\text{  otherwise} \end{cases}
\end{align*}
where $\delta$ is drawn from a normal distribution centered at $3$ with variance $0.1$. We plot the performance under the expected reward in Figure \ref{fig:experiment} (right). It can also be seen that adding entropy surpasses the non-regularized policy $\epsilon = 0$ and that increasing $\epsilon$ higher will worsen performance (as seen by $\epsilon = 2.5$).
\section{Conclusion}
Our results allow us to reason about regularization of policies and the regression Q-learning objective from the perspective of robustness. This is not surprising given the advancements in machine learning more generally pointing at the link between regularization and robustness along with the impressive empirical evidence of these schemes. Regularized objectives, however, offer other benefits that are inherently sample based phenomenon such as smoothened objectives or stable training. While our results do not directly target this, we have built a connection between two objectives which will pose modular for future developments.
\section*{Acknowledgements}
We would like to thank Zakaria Mhammedi for useful feedback regarding the technical analysis.

\bibliographystyle{apalike}
\bibliography{references}

\clearpage
\newpage
\section{Proofs of Main Results}\label{supp-formal}
We first introduce some notation that will be used exclusively for the Appendix. For any function $R: \mathscr{B}(\mathcal{X}) \to \mathbb{R}$, we define $R_{+}(\mu) = R(\mu) + \iota_{\mathscr{P}}(\mu)$ and $R_{-}(\mu) = R(\mu) - \iota_{\mathscr{P}}(\mu)$. Indeed, it should noted that if $R$ is upper semi-continuous concave then $R_{-}$ is upper  semi-continuous concave and $-R_{-}$ is proper convex. The central benefit of rewriting $R$ in this is way is due to
\begin{align*}
    \sup_{\mu \in \mathcal{K}_{P,\gamma}} R(\mu) = \sup_{\mu \in \mathcal{K}_{P,\gamma}} R_{-}(\mu).
\end{align*}
First we will show a technical result.
\begin{lemma}
If $R: \mathscr{B}(\mathcal{X}) \to \mathbb{R}$ is upper semicontinuous and concave then $(-R_{-})^{\star}$ is increasing.
\end{lemma}
\begin{proof}
Let $r,r' \in \mathcal{F}_b(\mathcal{X})$ such that  $r \leq r'$ and let
\begin{align*}
    \nu \in \argsup_{\mu \in \mathscr{P}(\mathcal{X})} \bracket{ \int_{\mathcal{X}} r(x) d\mu(x) + R(\mu)},
\end{align*}
noting that $\nu$ exists since the mapping $\mu \mapsto \int_{\mathcal{X}} r(x) d\mu(x) + R(\mu)$ is concave, upper semicontinuous and $\mathscr{P}(\mathcal{X})$ is compact. Next we have
\begin{align*}
    &(-R_{-})^{\star}(r) - (-R_{-})^{\star}(r')\\ &= \sup_{\mu \in \mathscr{P}(\mathcal{X})} \bracket{ \int_{\mathcal{X}} r(x) d\mu(x) + R(\mu)} -  \sup_{\mu \in \mathscr{P}(\mathcal{X})} \bracket{ \int_{\mathcal{X}} r'(x) d\mu(x) + R(\mu)}\\
    &\leq \int_{\mathcal{X}} r(x) d\nu(x) + R(\nu) - \int_{\mathcal{X}} r'(x) d\nu(x) - R(\nu)\\
    &= \int_{\mathcal{X}} \bracket{r(x) - r'(x)} d\nu(x)\\
    &\leq 0
\end{align*}
\end{proof}
We also recall some classical results regarding Fenchel duality between the spaces $\mathcal{F}_b(\mathcal{X})$ and $\mathscr{B}(\mathcal{X})$.
\begin{definition}[\cite{rockafellar1968integrals}]
For any proper convex function $F: \mathcal{F}_b(\mathcal{X}) \to (-\infty, \infty]$ and $\mu \in \mathscr{B}(\mathcal{X})$ we define
\begin{align*}
    F^{\star}(\mu) = \sup_{h \in \mathcal{F}_b} \bracket{ \int_{\mathcal{X}} h d\mu - F(h) }
\end{align*}
and for any $h \in \mathcal{F}_b(\Omega)$ we define
\begin{align*}
    F^{\star \star}(h) = \sup_{\mu \in \mathscr{B}(\mathcal{X})} \bracket{ \int_{\mathcal{X}} h d\mu - F^{\star}(\mu) }.
\end{align*}
\end{definition}
\begin{theorem}[\cite{zalinescu2002convex} Theorem 2.3.3]
\label{self-conjugacy}
If $X$ is a Hausdorff locally convex space, and $F: X \to (-\infty, \infty]$ is a proper convex lower semi-continuous function then $F^{\star \star} = F$.
\end{theorem}
\subsection{Proof of Theorem \ref{value-theorem}}
We have
\begin{align*}
    \sup_{\mu\in \mathcal{K}_{P,\gamma} } R(\mu) &= \sup_{\mu\in \mathcal{K}_{P,\gamma} } -\bracket{-R(\mu)}\\
    &\stackrel{(1)}{=} \sup_{\mu\in \mathcal{K}_{P,\gamma} } -\bracket{-R(\mu)}^{\star \star}\\
    &\stackrel{(2)}{=} \sup_{\mu\in \mathcal{K}_{P,\gamma} } -\sup_{r' \in \mathcal{F}_b(\mathcal{X})} \bracket{\int_{\mathcal{X}} r'(x) d\mu(x) - \bracket{-R}^{\star}(r') }\\
    &= \sup_{\mu\in \mathcal{K}_{P,\gamma} } \inf_{r' \in \mathcal{F}_b(\mathcal{X})} \bracket{\int_{\mathcal{X}} \bracket{-r'(x)} d\mu(x) + \bracket{-R}^{\star}(r') }\\
    &\stackrel{(3)}{=}  \inf_{r' \in \mathcal{F}_b(\mathcal{X})} \sup_{\mu\in \mathcal{K}_{P,\gamma} } \bracket{\int_{\mathcal{X}} \bracket{-r'(x)} d\mu(x) + \bracket{-R}^{\star}(r') }\\
    &\stackrel{(4)}{=}  \inf_{r' \in \mathcal{F}_b(\mathcal{X})}  \bracket{ \sup_{\mu\in \mathcal{K}_{P,\gamma} }\int_{\mathcal{X}} r'(x) d\mu(x) + \bracket{-R}^{\star}(-r') }\\
    &\stackrel{(5)}{=}  \inf_{r' \in \mathcal{F}_b(\mathcal{X})}  \bracket{ \operatorname{RL}_{P,\gamma}(r') + \bracket{-R}^{\star}(-r') }
    \end{align*}
where $(1)$ holds since $-R$ is proper convex, $(2)$ is the definition of the conjugate, $(3)$ is an application of Ky Fan's minimax theorem \citep[Theorem~2]{fan1953minimax} noting that the set $\mathcal{K}_{P,\gamma}$ is compact, and that the mapping $r\mapsto \int_{\mathcal{X}} \bracket{-r'(x)} d\mu(x) + \bracket{-R}^{\star}(r')$ is concave and the mapping $\mu \mapsto \int_{\mathcal{X}} \bracket{-r'(x)} d\mu(x)$ is linear. $(4)$ holds by negating $r'$ since $-\mathcal{F}_b(\mathcal{X}) = \mathcal{F}_b(\mathcal{X})$ and $(5)$ holds by definition.
\subsection{Proof of Theorem \ref{variable-theorem}}
By definition, we have $\operatorname{RL}_{P,\gamma}(r^{*}) - \ip{r^{*}}{\mu^{*}} \geq 0$. To show the other direction, it follows that
\begin{align*}
    \operatorname{RL}_{P,\gamma}(r^{*}) -  \ip{r^{*}}{\mu^{*}} &= \bracket{\operatorname{RL}_{P,\gamma}(r^{*}) + (-R)^{\star}(-r^{*}) } -  \bracket{\ip{r^{*}}{\mu^{*}} +  (-R)^{\star}(-r^{*})}\\
    &\stackrel{(1)}{=} \inf_{r' \in \mathcal{F}_b(\mathcal{X})} \bracket{ \operatorname{RL}_{P,\gamma}(r') +(-R)^{\star}(-r')} - \bracket{\ip{r^{*}}{\mu^{*}} +  (-R)^{\star}(-r^{*})}\\
    &\stackrel{(2)}{=} \sup_{\mu \in \mathcal{K}_{P,\gamma}} R(\mu) - \bracket{\ip{r^{*}}{\mu^{*}} +  (-R)^{\star}(-r^{*})}\\
    &\stackrel{(3)}{=} R(\mu^{*}) - \bracket{\ip{r^{*}}{\mu^{*}} +  (-R)^{\star}(-r^{*})}\\
    &= \ip{-r^{*}}{\mu^{*}} - \bracket{-R}(\mu^{*}) - \bracket{-R}^{\star}(-r^{*})\\
    &\stackrel{(4)}{\leq} 0,
\end{align*}
where $(1)$ follows via optimality of $r^{*}$, $(2)$ is due to the duality result, $(3)$ follows via optimality of $\mu^{*}$ and $(4)$ is an application of the Fenchel-Young inequality on the convex function $-R$. Finally, we have $\operatorname{RL}_{P,\gamma}(r^{*}) = \ip{r^{*}}{\mu^{*}}$, which implies optimality of $\mu^{*}$ and concludes the proof.

\subsection{Proof of Theorem \ref{generalized-valueproblem}}
Using the classic linear programming duality result, we have
\begin{align}
\operatorname{RL}_{P,\gamma}(r) = (1 - \gamma) \inf_{V \in \mathcal{V}_{P,r,\gamma}} \int_{\mathcal{S}}V(s)d\mu_0(s), \label{LP-duality}
\end{align}
where
\begin{align*}
    \mathcal{V}_{P,r,\gamma} = \braces{V \in \mathcal{F}_b(\mathcal{S}) : V(s) \geq r(s,a) + \gamma \int_{\mathcal{S}} V(s') dP(s' \mid s,a), \forall (s,a) \in \mathcal{X} },
\end{align*}
and define
\begin{align}
    r_V(s,a) := V(s) - \gamma \int_{\mathcal{S}} V(s') dP(s' \mid s,a).
\end{align}
It then holds that
\begin{align*}
    \sup_{\mu \in \mathcal{K}_{P,\gamma}} R(\mu) &\stackrel{(1)}{=} \inf_{r' \in \mathcal{F}_b(\mathcal{X})}\bracket{\operatorname{RL}_{P,\gamma}(r') + (-R)^{\star}(-r') }\\
    &\stackrel{(2)}{=} \inf_{r' \in \mathcal{F}_b(\mathcal{X})}\bracket{(1 - \gamma)\inf_{V \in \mathcal{V}_{P,r',\gamma}} \int_{\mathcal{S}}V(s)d\mu_0(s) + (-R)^{\star}(-r') }\\
    &= \inf_{r' \in \mathcal{F}_b(\mathcal{X})}\inf_{V \in \mathcal{F}_b(\mathcal{S})} \bracket{(1 - \gamma) \int_{\mathcal{S}}V(s)d\mu_0(s) + (-R)^{\star}(-r') + \iota_{\mathcal{V}_{P,r',\gamma}}(V) }\\
    &= \inf_{V \in \mathcal{F}_b(\mathcal{S})}\inf_{r' \in \mathcal{F}_b(\mathcal{X})}\bracket{ (1 - \gamma)\int_{\mathcal{S}}V(s)d\mu_0(s) + (-R)^{\star}(-r') + \iota_{\mathcal{V}_{P,r',\gamma}}(V) }\\
    &= \inf_{V \in \mathcal{F}_b(\mathcal{S})}\inf_{r' \leq r_V}\bracket{ (1 - \gamma)\int_{\mathcal{S}}V(s)d\mu_0(s) + (-R)^{\star}(-r')  },
\end{align*}
where $(1)$ is due to Theorem \ref{value-theorem},  $(2)$ is due to \eqref{LP-duality} and noting that $r \leq r_V$ implies $V(s) \geq r(s,a) + \gamma \int_{\mathcal{S}} V(s') dP(s' \mid s,a)$ concludes the proof.
\subsection{Proof of Lemma \ref{lemma:optVandoptr}}
First note that for any $\mu \in \mathcal{K}_{P,\gamma}$, we have
\begin{align*}
    &\int_{\mathcal{X}} r_V(s,a) d\mu(s,a)\\ &= \bracket{\int_{\mathcal{S}} V(s) d\mu(s,a) - \gamma \int_{\mathcal{X}} \int_{\mathcal{S}} V(s') dP(s'\mid s,a) d\mu(s,a) }\\
    &=  \bracket{\int_{\mathcal{S}} V(s) d\mu(s,a) - \int_{\mathcal{S}} V(s) d\mu(s,a) + (1 - \gamma) \int_{\mathcal{S}} V(s) d\mu_0(s) }\\
    &= (1 - \gamma) \int_{\mathcal{S}} V(s) d\mu_0(s),
\end{align*}
and so we can conclude for any $V \in \mathcal{F}_b(\mathcal{S})$, we have
\begin{align*}
    \operatorname{RL}_{P,\gamma}(r_V) = (1 - \gamma) \int_{\mathcal{S}} V(s) d\mu_0(s).
\end{align*}
Next, we have
\begin{align*}
     \sup_{\mu \in \mathcal{K}_{P,\gamma}} R(\mu) &= \inf_{V \in \mathcal{F}_b(\mathcal{S})}\bracket{ (1 - \gamma) \int_{\mathcal{S}}V(s)d\mu_0(s) +(-R)^{\star}(-r_V)  }\\
     &= \inf_{V \in \mathcal{F}_b(\mathcal{S})}\bracket{ \operatorname{RL}_{P,\gamma}(r_V) +(-R)^{\star}(-r_V)  }\\
     &\geq \inf_{r' \in \mathcal{F}_b(\mathcal{X})}\bracket{ \operatorname{RL}_{P,\gamma}(r') + (-R)^{\star}\bracket{-r'} }\\
    &= \sup_{\mu \in \mathcal{K}_{P,\gamma}} R(\mu),
\end{align*}
and since the lower bound can achieve equality, it implies that the optimal $r^{*}$ is of the form $r_V$.

\subsection{Proof of Corollary \ref{value-corollary}}
We have
\begin{align*}
    (-R)^{\star}(-r') &= \sup_{\mu \in \mathscr{B}(\mathcal{X})} \bracket{\int_{\mathcal{X}} -r'(x) d\mu(x) + R(\mu) }\\
    &= \sup_{\mu \in \mathscr{B}(\mathcal{X})} \bracket{\int_{\mathcal{X}} -r'(x) d\mu(x) + \int_{\mathcal{X}} r(x) d\mu(x) - \epsilon\Omega(\mu) }\\
    &= \sup_{\mu \in \mathscr{B}(\mathcal{X})} \bracket{\int_{\mathcal{X}} r(x)-r'(x) d\mu(x)  - \epsilon\Omega(\mu) }\\
    &= \epsilon \sup_{\mu \in \mathscr{B}(\mathcal{X})} \bracket{\int_{\mathcal{X}} \frac{r(x) - r'(x)}{\epsilon} d\mu(x)  - \Omega(\mu) }\\
    &= \epsilon \Omega^{\star}\bracket{\frac{r - r'}{\epsilon}},
\end{align*}
which concludes the proof.
\subsection{Proof of Theorem \ref{thm:smoothQlearning}}
First define the set
\begin{align*}
    \mathcal{Q}_{P,r,\gamma} = \braces{Q \in \mathcal{F}_b(\mathcal{X}) : Q(s,a) \geq r(s,a) + \gamma \int_{\mathcal{X}} \sup_{a' \in \mathcal{A}} Q(s',a') dP(s' \mid s,a) },
\end{align*}
and define
\begin{align*}
    r_Q(s,a) = Q(s,a) - \gamma  \int_{\mathcal{X}}\sup_{a' \in \mathcal{A}} Q(s',a') dP(s' \mid s,a) 
\end{align*}
Next we can write
\begin{align} 
    \operatorname{RL}_{P,\gamma}(r) = \inf_{Q \in \mathcal{Q}_{P,r,\gamma}}  \int_{\mathcal{S}} \sup_{a \in \mathcal{A}} Q(s,a) d\mu_0(s),\tag{A} \label{eq:rl-Q}
\end{align}
next we have
\begin{align*}
    \sup_{\mu \in \mathcal{K}_{P,\gamma}} R(\mu) &\stackrel{(1)}{=} \inf_{r' \in \mathcal{F}_b(\mathcal{X})} \bracket{\operatorname{RL}_{P,\gamma}(r') + (-R)^{\star}(-r')  }\\
    &\stackrel{(2)}{=} \inf_{r' \in \mathcal{F}_b(\mathcal{X})} \bracket{ \inf_{Q \in \mathcal{Q}_{P,r',\gamma}} \int_{\mathcal{S}} \sup_{a \in \mathcal{A}} Q(s,a) d\mu_0(s) + (-R)^{\star}(-r')  }\\
    &= \inf_{r' \in \mathcal{F}_b(\mathcal{X})} \bracket{ \inf_{Q \in \mathcal{F}_b(\mathcal{X})} \bracket{\int_{\mathcal{S}} \sup_{a \in \mathcal{A}} Q(s,a) d\mu_0(s) + \iota_{\mathcal{Q}_{P,r',\gamma}}(Q)} +  (-R)^{\star}(-r')  }\\
    &= \inf_{r' \in \mathcal{F}_b(\mathcal{X})} \inf_{Q \in \mathcal{F}_b(\mathcal{X})} \bracket{  \int_{\mathcal{S}} \sup_{a \in \mathcal{A}} Q(s,a) d\mu_0(s)  +  (-R)^{\star}(-r') + \iota_{\mathcal{Q}_{P,r',\gamma}}(Q)  }\\
    &=  \inf_{Q \in \mathcal{F}_b(\mathcal{X})} \inf_{r' \in \mathcal{F}_b(\mathcal{X})} \bracket{  \int_{\mathcal{S}} \sup_{a \in \mathcal{A}} Q(s,a) d\mu_0(s)  +  (-R)^{\star}(-r') + \iota_{\mathcal{Q}_{P,r',\gamma}}(Q)  }\\
    &=  \inf_{Q \in \mathcal{F}_b(\mathcal{X})}  \bracket{  \int_{\mathcal{S}} \sup_{a \in \mathcal{A}} Q(s,a) d\mu_0(s)  + \inf_{r' \in \mathcal{F}_b(\mathcal{X})}\bracket{ (-R)^{\star}(-r') + \iota_{\mathcal{Q}_{P,r',\gamma}}(Q) } }\\
    &=  \inf_{Q \in \mathcal{F}_b(\mathcal{X})}  \bracket{  \int_{\mathcal{S}} \sup_{a \in \mathcal{A}} Q(s,a) d\mu_0(s)  + \inf_{r' \leq r_{Q} } (-R)^{\star}(-r') }\\
    &\stackrel{(3)}{=}  \inf_{Q \in \mathcal{F}_b(\mathcal{X})}  \bracket{  \int_{\mathcal{S}} \sup_{a \in \mathcal{A}} Q(s,a) d\mu_0(s)  + (-R)^{\star}(-r_Q) },
\end{align*}
where $(1)$ is due to Theorem \ref{value-theorem}, $(2)$ is due to \eqref{eq:rl-Q}, and $(3)$ follows since $(-R)^{\star}$ is increasing by assumption. Next, noting that $(-R)^{\star}(-r_Q) = \epsilon\Omega^{\star}\bracket{\frac{r - r_Q}{\epsilon}}$, and that
\begin{align*}
    r - r_Q &= r(s,a) - \frac{Q(s,a)}{1 - \gamma} +\gamma \int_{\mathcal{X}} \sup_{a' \in \mathcal{A}} Q(s',a') dP(s' \mid s,a)  \\
            &=  \bracket{r(s,a) + \gamma \int_{\mathcal{X}} \sup_{a' \in \mathcal{A}} Q(s',a') dP(s' \mid s,a)} - Q(s,a) \\
            &= \mathcal{T}Q - Q,
\end{align*}
which is the difference between the Bellman operator. Putting this together yields 
\begin{align*}
    &\sup_{\mu \in \mathcal{K}_{P,\gamma}} R(\mu)\\ &= \inf_{Q \in \mathcal{F}_b(\mathcal{X})}  \bracket{\epsilon\Omega^{\star}\bracket{\frac{r - r_Q}{\epsilon}} +   \int_{\mathcal{S}} \sup_{a \in \mathcal{A}} Q(s,a) d\mu_0(s)   }\\
    &= \inf_{Q \in \mathcal{F}_b(\mathcal{X})}  \bracket{\epsilon\Omega^{\star}\bracket{\frac{\mathcal{T} Q - Q}{\epsilon}} +   \int_{\mathcal{S}} \sup_{a \in \mathcal{A}} Q(s,a) d\mu_0(s)   }
\end{align*}

\subsection{Proof of Lemma \ref{SAC-conjugate}}
We first set $n = \card{A}$. Let $\mathcal{F}_b(\mathcal{S}, \mathbb{R}^n)$ denote the set of measurable and bounded functions mapping from $\mathcal{S}$ into $\mathbb{R}^n$. For any $\pi \in \mathcal{F}_b(\mathcal{S}, \mathbb{R}^n)$, we use $\pi(a \mid s)$ to denote the index corresponding to $a \in \mathcal{A}$ for the function $\pi$ evaluated at $s \in \mathcal{S}$. Next, we define the following set:
\begin{align*}
    \mathcal{B}_{\times} := \braces{ \mu(s,a) = \pi(a \mid s) \cdot \mu_S(s) \mid \mu_S \in \mathscr{P}(\mathcal{S}), \pi \in \mathcal{F}_b(\mathcal{S}, \mathbb{R}^n) }, 
\end{align*}
noting that $\mathcal{B}_{\times} \subseteq \mathscr{B}(\mathcal{X})$. We also have that $\mathscr{P}(\mathcal{X}) \subset \mathcal{B}_{\times}$ since this corresponds to having each $\pi(a \mid s)$ satisfy $\pi(a \mid s) \in [0,1]$ and $\sum_{a \in \mathcal{A}} \pi(a \mid s) = 1$. We then redefine 
\begin{align*}
    \Omega(\mu) = \begin{cases} \E_{\mu(s,a)}\left[\operatorname{KL}(\pi_{\mu}(\cdot \mid s),U)\right] &\text{  if }\mu \in \mathcal{B}_{\times} \\ \infty &\text{  if }\mu \notin \mathcal{B}_{\times} \end{cases}
\end{align*}
We will first show that this choice of $\Omega$ is convex. First we need a Lemma that will make it easier.
\begin{lemma}
The functional $F: \mathbb{R}^n \to \mathbb{R}$ defined as
\begin{align*}
    F(\mathbf{x}) = \sum_{i=1}^n x_i \cdot \log\bracket{\frac{x_i}{\sum_{j=1}x_j}}
\end{align*}
is convex over its domain $\mathbb{R}_{> 0}^n$.
\end{lemma}
\begin{proof}
We derive the Hessian of $F$ which can be verified to be:
\begin{align*}
    HF(\mathbf{x}) = \operatorname{diag}\bracket{\frac{1}{x_1}, \frac{1}{x_2}, \ldots, \frac{1}{x_n}}- \frac{1}{\sum_{i=1}^n x_i} \cdot \mathbf{1}^{\intercal}\mathbf{1}.
\end{align*}
Next, we have for any vector $z \in \mathbb{R}^n$ and $x \in \operatorname{dom}F$:
\begin{align*}
    z^{\intercal}HF(x)z &= z^{\intercal} \operatorname{diag}\bracket{\frac{1}{x_1}, \frac{1}{x_2}, \ldots, \frac{1}{x_n}} z - \frac{1}{\sum_{i=1}^n x_i}  \bracket{\sum_{i=1}^n z_i}^2\\
    &= \sum_{i=1}^n \frac{z_i^2}{x_i} - \frac{1}{\sum_{i=1}^n x_i}  \bracket{\sum_{i=1}^n z_i}^2\\
    &= \frac{1}{\sum_{i=1}^n x_i} \bracket{\bracket{\sum_{i=1}^n x_i} \cdot \bracket{\sum_{i=1}^n \frac{z_i^2}{x_i}} -   \bracket{\sum_{i=1}^n z_i}^2}\\
    &\geq 0,
\end{align*}
where the last inequality follows by an application of Cauchy-Schwarz inequality noting that $x \in \operatorname{Dom}F = \mathbb{R}_{> 0}^n$. Since the Hessian is positive semi-definite, it follows that $F$ is convex.
\end{proof}
First denote by $\mu_S(s) = \sum_{a \in \mathcal{A}} \mu(s,a)$ and note that $\pi_{\mu}(a \mid s) = \mu(s,a) / \mu_S(s)$. For any $\mu \in \operatorname{dom} \Omega$, we have
\begin{align*}
    \Omega(\mu) &= \E_{\mu(s,a)} \left[\operatorname{KL}(\pi_{\mu}, U)\right]\\
                &= \E_{\mu(s,a)} \left[ \sum_{a \in \mathcal{A}} \pi_{\mu}(a \mid s) \cdot \log\bracket{\pi_{\mu}(a \mid s)} + \log n \right]\\
                &= \E_{\mu_S(s)} \left[ \sum_{a \in \mathcal{A}} \pi_{\mu}(a \mid s) \cdot \log\bracket{\pi_{\mu}(a \mid s)}  \right] + \log n\\
                &=  \int_{\mathcal{S}} \sum_{a \in \mathcal{A}} \mu_S(s) \pi_{\mu}(a \mid s) \cdot \log\bracket{\pi_{\mu}(a \mid s)}  ds + \log n\\
                &=  \int_{\mathcal{S}} \sum_{a \in \mathcal{A}} \mu(s,a)\cdot \log\bracket{\frac{\mu(s,a)}{\sum_{a' \in \mathcal{A}}\mu(s,a') }}  ds + \log n,
\end{align*}
and convexity follows by the above Lemma. Before we proceed, we need to also show that $\mathcal{B}_{\times}$ is convex so that our redefining of $\Omega$ does not break convexity established above. Consider $\mu, \nu \in \mathcal{B}_{\times}$ and so there exists $\mu_S,\nu_S \in \mathscr{P}(\mathcal{S})$ and $\pi_{\mu}, \pi_{\nu} \in \mathcal{F}_b(\mathcal{S}, \mathbb{R}^n)$ with $\mu(s,a) = \pi_{\mu} (a \mid s) \cdot \mu_S(s)$ and $\nu(s,a) = \pi_{\nu}(a \mid s) \cdot \nu_S(s)$. For any $\lambda \in [0,1]$, we have (setting $P_{\mu,\nu}(s) = \frac{\mu_S(s) + \nu_S(s)}{2}$)
\begin{align*}
    \lambda \cdot \mu(s,a) + (1 - \lambda) \nu(s,a) &= \lambda \pi_{\mu} (a \mid s) \cdot \mu_S(s) + (1 - \lambda) \cdot \pi_{\nu}(a \mid s) \cdot \nu_S(s)\\
    &= P_{\mu,\nu}(s) \cdot \bracket{\lambda \pi_{\mu} (a \mid s) \cdot \frac{\mu_S(s)}{P_{\mu,\nu}(s)} + (1 - \lambda) \cdot \pi_{\nu}(a \mid s) \cdot \frac{\nu_S(s)}{P_{\mu,\nu}(s)}}.
\end{align*}
By construction, both $\mu_S$ and $\nu_S$ are absolutely continuous with respect to $P_{\mu,\nu}$ and thus the terms inside the bracket are bounded and well-defined. Moreover $P_{\mu,\nu} \in \mathscr{P}(\mathcal{S})$ and thus this element is in $\mathcal{B}_{\times}$, which concludes the convexity proof. We now proceed to derive the conjugate. For any $r' \in \mathcal{F}_b(\mathcal{X})$ we have
\begin{align*}
    \Omega^{\star}(r') &= \sup_{\mu \in \mathscr{B}(\mathcal{X})} \bracket{ \int_{\mathcal{X}} r'(s,a) d\mu(s,a) - \Omega(\mu)  }\\
    &\stackrel{(1)}{=} \sup_{\mu \in \mathcal{B}_{\times}} \bracket{ \int_{\mathcal{X}} r'(s,a) d\mu(s,a) - \Omega(\mu)  }\\
    &= \sup_{\mu \in \mathcal{B}_{\times}} \bracket{ \int_{\mathcal{X}} r'(s,a) d\mu(s,a) - \E_{\mu(s,a)}\left[\operatorname{KL}(\pi_{\mu}(\cdot \mid s),U)\right] }\\
    &= \sup_{\mu \in \mathcal{B}_{\times}} \bracket{ \int_{\mathcal{X}} \bracket{ \int_{\mathcal{A}} r'(s,a) d\pi_{\mu}(a \mid s) - \operatorname{KL}(\pi_{\mu}(\cdot \mid s),U)} d\mu(s,a)   }\\
    &= \sup_{\mu_S \in \mathscr{P}(\mathcal{S})} \sup_{\pi_{\mu}(\cdot \mid s) \in \mathcal{F}_b(\mathcal{S}, \mathbb{R}^n)} \bracket{ \int_{\mathcal{X}} \bracket{ \int_{\mathcal{A}} r'(s,a) d\pi_{\mu}(a \mid s) - \operatorname{KL}(\pi_{\mu}(\cdot \mid s),U)} d\mu_S(s)   }\\
    &\stackrel{(2)}{=} \sup_{\mu_S \in \mathscr{P}(\mathcal{S})} \int_{\mathcal{X}}  \sup_{\pi_{\mu} \in \mathbb{R}^n }\bracket{  \int_{\mathcal{A}}  r'(s,a) d\pi_{\mu}(a) - \operatorname{KL}(\pi_{\mu},U)} d\mu_S(s) \\
    &\stackrel{(3)}{=} \sup_{\mu_S \in \mathscr{P}(\mathcal{S})} \int_{\mathcal{X}}  \sup_{\pi_{\mu} \in \mathscr{P}(\mathcal{A}) }\bracket{  \int_{\mathcal{A}}  r'(s,a) d\pi_{\mu}(a) - \operatorname{KL}(\pi_{\mu},U)} d\mu_S(s) \\
    &\stackrel{(4)}{=} \sup_{\mu_S \in \mathscr{P}(\mathcal{S})} \int_{\mathcal{X}} \exp\bracket{r'(s,a)} dU(a) - 1\\
    &\stackrel{(5)}{=} \sup_{s \in \mathcal{S}} \int_{\mathcal{X}} \exp\bracket{r'(s,a)} dU(a) - 1, 
\end{align*}
where $(1)$ holds since $\operatorname{dom} \Omega \subseteq \mathcal{B}_{\times}$. $(2)$ holds from \citep[Theorem~14.60, p. 677]{rockafellar2009variational} using the fact that $\mathcal{F}_b(\mathcal{S},\mathbb{R}^n)$ is trivially a decomposable space in definition \citep[Definition~14.59, p. 676]{rockafellar2009variational}. $(3)$ holds since $\operatorname{dom}\bracket{\operatorname{KL}(\cdot, U)} \subseteq \mathscr{P}(\mathcal{A}) \subset \mathbb{R}^n$. $(4)$ is due to \citep[Proposition~5]{feydy2019interpolating} and $(5)$ follows by noting that the optimal $\mu_S$ is concentrated around the supremum.
\subsection{Imitation Learning}
\subsubsection{$f$-divergence}
Note that for any $r \in \mathcal{F}_b(\mathcal{X})$ we have
\begin{align*}
    (-R)^{\star}(r) &= \sup_{\nu \in \mathscr{B}(\mathcal{X})} \bracket{\int_{\mathcal{X}}r(x) d\nu(x) + R(\nu) }\\
    &= \sup_{\nu \in \mathscr{B}(\mathcal{X})} \bracket{\int_{\mathcal{X}}r(x) d\nu(x) - \operatorname{KL}(\nu,\mu_E}\\
    &\stackrel{(1)}{=} \int_{\mathcal{X}} r(x) d\mu_E(x) - 1,
\end{align*}
where $(1)$ holds due to \cite[Proposition~5]{feydy2019interpolating}. We will now show that $(-R)^{\star}$ is increasing for any $R(\mu) = -D_f(\mu,\mu_E)$ where $D_f$ is an $f$-divergence. First let
\begin{align*}
    \nu \in \argsup_{\mu \in \mathscr{P}(\mathcal{X})} \bracket{ \int_{\mathcal{X}} r(x) d\mu(x) + R(\mu)},
\end{align*}
noting that $\nu$ exists since the mapping $\mu \mapsto \int_{\mathcal{X}} r(x) d\mu(x) + R(\mu)$ is concave, upper semicontinuous and $\mathscr{P}(\mathcal{X})$ is compact. For any $r' \geq r$
\begin{align*}
    &(-R_{-})^{\star}(r) - (-R_{-})^{\star}(r')\\ &= \sup_{\mu \in \mathscr{B}(\mathcal{X})} \bracket{ \int_{\mathcal{X}} r(x) d\mu(x) + R(\mu)} -  \sup_{\mu \in \mathscr{B}(\mathcal{X})} \bracket{ \int_{\mathcal{X}} r'(x) d\mu(x) + R(\mu)} \\&\stackrel{(1)}{=} \sup_{\mu \in \mathscr{P}(\mathcal{X})} \bracket{ \int_{\mathcal{X}} r(x) d\mu(x) + R(\mu)} -  \sup_{\mu \in \mathscr{P}(\mathcal{X})} \bracket{ \int_{\mathcal{X}} r'(x) d\mu(x) + R(\mu)}\\
    &\leq \int_{\mathcal{X}} r(x) d\nu(x) + R(\nu) - \int_{\mathcal{X}} r'(x) d\nu(x) - R(\nu)\\
    &= \int_{\mathcal{X}} \bracket{r(x) - r'(x)} d\nu(x)\\
    &\leq 0,
\end{align*}
where $(1)$ holds due to the fact that $\operatorname{dom}\bracket{D_f(\cdot,\mu_E)} \subseteq \mathscr{P}(\mathcal{X})$.
\subsubsection{InfoGAIL}
In this case, we exploit the fact that $-R(\mu)$ takes the form of an Integral Probability Metric between $\mu$ and $\mu_E$. Let $\mathcal{H}_L$ the set of functions that are $L$-Lipschitz with respect to $d$. For any $r \in \mathcal{F}_b(\mathcal{X})$ we have
\begin{align*}
    (-R)^{\star}(r) &= \sup_{\nu \in \mathscr{B}(\mathcal{X})} \bracket{ \int_{\mathcal{X}} r(x) d\nu(x) - \sup_{h: \operatorname{Lip}_d(h) \leq L}\bracket{\int_{\mathcal{X}} h(x) d\nu(x) - \int_{\mathcal{X}} h(x) d\mu_E(x) }}\\
    &\stackrel{(1)}{=} \int_{\mathcal{X}} r(x) d\mu_E(x) + \iota_{\mathcal{H}_L}(r), \label{infogail-conjugate}
\end{align*}
where $(1)$ is due to \citep[Lemma~5]{husain2020distributional}. Thus, it holds that
\begin{align*}
    \sup_{\mu \in \mathcal{K}_{P,\gamma}} R(\mu) &= \inf_{r' \in \mathcal{F}_b(\mathcal{X})} \bracket{\operatorname{RL}_{P,\gamma}(r') + \int_{\mathcal{X}} -r'(x) d\mu_E(x) + \iota_{\mathcal{H}_L}(-r')}\\
    &\stackrel{(2)}{=} \inf_{r' \in \mathcal{F}_b(\mathcal{X})} \bracket{\operatorname{RL}_{P,\gamma}(r') - \int_{\mathcal{X}} r'(x) d\mu_E(x) + \iota_{\mathcal{H}_L}(r')}\\
    &= \inf_{r' : \operatorname{Lip}_d \leq L} \bracket{\operatorname{RL}_{P,\gamma}(r')- \int_{\mathcal{X}} r'(x) d\mu_E(x) } ,
\end{align*}
where $(2)$ holds since $\operatorname{Lip}_d(-r) = \operatorname{Lip}_d(r)$. We now show that adding an entropy term to 
\begin{align}
    R(\mu) =- \sup_{h : \operatorname{Lip}_d(h) \leq L} \bracket{\int_{\mathcal{X}}h(x)d\mu(x) - \int_{\mathcal{X}} h(x) d\mu_E(x) } - \epsilon \E_{\mu(s,a)}\left[\operatorname{KL}(\pi_{\mu}(\cdot \mid s), U_{A}) \right]
\end{align}
will ensure that $(-R)^{\star}$ is increasing. Using standard results from \citep{penot2012calculus} that the conjugate of the sum of two functions is the infimal convolution between their conjugates mean we will convolve both \eqref{infogail-conjugate} and entropy conjugate from Lemma 2 of the main file.:
\begin{align}
    (-R)^{\star}(r') &= \inf_{r \in \mathcal{F}_b(\mathcal{X})} \bracket{ \sup_{s \in \mathcal{S}} \int_{\mathcal{X}} \exp\bracket{r'(s,a) - r(s,a)} dU(a) + \int_{\mathcal{X}} r d\mu_E + \iota_{\mathcal{H}_L}(r) }\\
    &= \inf_{r \in \mathcal{H}_{L}} \bracket{ \sup_{s \in \mathcal{S}} \int_{\mathcal{X}} \exp\bracket{r'(s,a) - r(s,a)} dU(a) + \int_{\mathcal{X}} r d\mu_E}.
\end{align}
Let $r'' \leq r'$ pointwise and define
\begin{align}
    r^{\ast} \in \arginf_{r \in \mathcal{H}_{L}} \bracket{ \sup_{s \in \mathcal{S}} \int_{\mathcal{X}} \exp\bracket{r'(s,a) - r(s,a)} dU(a) + \int_{\mathcal{X}} r d\mu_E},
\end{align}
noting that since exists due to Weierstrass Theorem since $\mathcal{H}_L$ is compact and the mapping inside is convex and lower semicontinuous. Next, we have
\begin{align}
    &(-R)^{\star}(r'') - (-R)^{\star}(r')\\ &= \inf_{r \in \mathcal{H}_{L}} \bracket{ \sup_{s \in \mathcal{S}} \int_{\mathcal{X}} \exp\bracket{r''(s,a) - r(s,a)} dU(a) + \int_{\mathcal{X}} r d\mu_E}\\ &- \inf_{r \in \mathcal{H}_{L}} \bracket{ \sup_{s \in \mathcal{S}} \int_{\mathcal{X}} \exp\bracket{r'(s,a) - r(s,a)} dU(a) + \int_{\mathcal{X}} r d\mu_E}\\
    &\leq \sup_{s \in \mathcal{S}} \int_{\mathcal{X}} \exp\bracket{r''(s,a) - r^{\ast}(s,a)} dU(a) + \int_{\mathcal{X}} r^{\ast} d\mu_E\\ &- \sup_{s \in \mathcal{S}} \int_{\mathcal{X}} \exp\bracket{r'(s,a) - r^{\ast}(s,a)} dU(a) - \int_{\mathcal{X}} r^{\ast} d\mu_E\\
    &= \sup_{s \in \mathcal{S}} \int_{\mathcal{X}} \exp\bracket{r''(s,a) - r^{\ast}(s,a)} dU(a) - \sup_{s \in \mathcal{S}} \int_{\mathcal{X}} \exp\bracket{r'(s,a) - r^{\ast}(s,a)} dU(a)\\
    &\leq 0,
\end{align}
where the last inequality follows from the fact that $r'' \leq r'$ and thus this proves that $(-R)^{\star}$ is increasing.

\subsection{Entropic Exploration}
For any $r \in \mathcal{F}_b(\mathcal{X})$
\begin{align*}
    (-R)^{\star}(r) &= \sup_{\mu \in \mathscr{B}(\mathcal{X})} \bracket{\int_{\mathcal{X}}r(x) d\mu(x) - \operatorname{KL}(\mu,U_{\mathcal{X}}) }\\
    &\stackrel{(1)}{=} \int_{\mathcal{X}} \exp\bracket{r(x)} dU_{\mathcal{X}}(x) - 1,
\end{align*}
where $(1)$ follows from \citep[Proposition~5]{feydy2019interpolating}.

\end{document}